\documentclass[12pt]{article}

\usepackage[colorlinks, linkcolor=blue, citecolor=blue]{hyperref}           
\usepackage{graphicx,subfigure,amsmath,amssymb,amsfonts,bm,epsfig,epsf,url,dsfont}
\usepackage{amsthm}
\usepackage{tikz}
\usepackage{multirow}
\usepackage{bbm}      
\usepackage{booktabs}
\usepackage{cases}
\usepackage{fullpage}
\usepackage[small,bf]{caption}
\usepackage{natbib}
\usepackage[top=1in,bottom=1in,left=1in,right=1in]{geometry}
\usepackage{fancybox}

\usepackage{hyperref}
\usepackage{url}
\usepackage{float}
\usepackage{amssymb}
\usepackage{amsmath}
\usepackage{amsthm}
\usepackage{enumitem}
\usepackage{caption}
\usepackage{multirow}
\usepackage{graphicx}

\newtheorem{proposition}{Proposition}

\usepackage{comment}

\def\int{\displaystyle\mathop {\mbox{\rm int}}}    

\DeclareUnicodeCharacter{2212}{\ensuremath{-}}

\usepackage{algorithm,algorithmic}

\usepackage{wrapfig}
\usepackage{multicol}

\usepackage{subcaption}
\makeatletter
\newcommand*{\rom}[1]{\expandafter\@slowromancap\romannumeral #1@}
\makeatother
\makeatletter
\def\thanks#1{\protected@xdef\@thanks{\@thanks
        \protect\footnotetext{#1}}}
\makeatother
\allowdisplaybreaks
\usepackage[textsize=tiny]{todonotes}
\usepackage{url}
\usepackage{cite}
\usepackage{amsmath,amssymb,amsfonts}
\usepackage{algorithmic}
\usepackage{graphicx}
\usepackage{textcomp}
\usepackage{hyperref}
\usepackage{listings}
\usepackage{multirow} 
\usepackage{colortbl}
\usepackage{wrapfig}
\newcommand {\myvec}[1] {{\mbox{\boldmath $#1$}}}
\DeclareMathOperator{\erf}{erf}
\usepackage{amsthm}

\title{FineGates: LLMs Finetuning with Compression using Stochastic Gates}

\author{
    Jonathan Svirsky\thanks{J. Svirsky and O. Lindenbaum are with the Faculty of Engineering, Bar Ilan University, Ramat Gan, Israel. Email: \texttt{svirskj@biu.ac.il}, \texttt{ofir.lindenbaum@biu.ac.il}} 
    \and
    Yehonathan Refael\thanks{Y. Refael is with the Department of Electrical Engineering-Systems, Tel Aviv University, Tel Aviv 6997801, Israel. Email: \texttt{refaelkalim@mail.tau.ac.il}} 
    \and
    Ofir Lindenbaum
}

\begin{document}

\maketitle

\begin{abstract}
Large Language Models (LLMs), with billions of parameters, present significant challenges for full finetuning due to the high computational demands, memory requirements, and impracticality of many real-world applications. When faced with limited computational resources or small datasets, updating all model parameters can often result in overfitting. To address this, lightweight finetuning techniques have been proposed, like learning low-rank adapter layers. These methods aim to train only a few additional parameters combined with the base model, which remains frozen, reducing resource usage and mitigating overfitting risks. In this work, we propose an adaptor model based on stochastic gates that simultaneously sparsify the frozen base model with task-specific adaptation. Our method comes with a small number of trainable parameters and allows us to speed up the base model inference with competitive accuracy. We evaluate it in additional variants by equipping it with additional low-rank parameters and comparing it to several recent baselines. Our results show that the proposed method improves the finetuned model accuracy comparatively to the several baselines and allows the removal of up to 20-40\% without significant accuracy loss. 
\end{abstract}
\section{Introduction}
Large language models (LLMs) have revolutionized natural language processing by enabling powerful and versatile applications across various tasks. These models, pre-trained on vast amounts of text data, possess a deep understanding of language, making them valuable for text generation, translation, and sentiment analysis tasks. However, finetuning is often necessary to tailor these models to specific applications. Finetuning allows the model to adapt to the nuances of a particular task by updating its parameters based on a smaller, task-specific dataset. The challenge is that users typically have limited data for finetuning, which can constrain the model's ability to optimize for the desired task fully and may lead to overfitting or suboptimal performance. Despite these challenges, finetuning remains crucial in leveraging the full potential of large language models for specialized applications.

Recently, several innovative methods have been proposed to optimize the finetuning process of large language models, addressing the challenges associated with limited data and the computational cost of updating all model parameters. One such approach is LoRA (Low-Rank Adaptation) \citep{hu2021lora}, which introduces a more efficient way to finetune models by freezing the original parameters and injecting trainable, low-rank matrices into each layer of the model. This technique significantly reduces the number of parameters that need to be updated during finetuning, making the process both faster and less resource-intensive. Most recent efforts in optimizing finetuning \citep{zhang2023lora, chavan2023one, xu2023qa, li2022parameter, lin2024nora, balazy2024lora, hu2021lora}, including LoRA, focus on adding new parameters while keeping the base model frozen, ensuring that the original pre-trained knowledge is retained \citep{rozner2024knowledge}.

When training models using low-rank adapters \citep{kopiczko2023vera, zhang2023lora, lin2024nora, balazy2024lora, hu2021lora}, the number of updated parameters is reduced, resulting in faster convergence during finetuning. However, the inference runtime remains the same as it still depends on the base model size. To treat this issue several methods were proposed large language models (LLM) compression like like quantization and pruning. Quantization reduces the memory usage of language models by converting their parameters into lower-bit data types \citep{bondarenko2024low, lin2024awq}. Although quantization reduces the memory consumption of language models, its speedup advantages rely on specialized framework support, which limits its flexibility and adaptability. 

Pruning \citep{ma2023llm} aims to improve inference efficiency in language models by removing unimportant parameters. Structured pruning \citep{xia2022structured} removes consistent blocks of parameters, or model dimensions, achieving more general inference efficiency improvements. This method often involves knowledge distillation \citep{hinton2015distilling}, increasing training costs. 

In this work, we propose a simple, efficient, and effective adaptation method that adopts the base model for the target task by learning stochastic gates on the weights of the base model with optional low-rank updates.  Instead of only adding low-rank matrices for adaptation, we train gates to preserve only task-specific information within the base model itself. This approach allows us to maintain the integrity of the pre-trained model while incorporating the essential nuances of each specific task directly into its structure, leading to a more effective finetuning process. Our method not only preserves task-specific information within the base model but also allows for a significant reduction in the adopted base model's layer weights. By efficiently embedding the necessary task-specific details while reducing unnecessary parameters, our method enhances both the effective performance and efficiency of the finetuned model, making it better suited for real-world applications where speed and resource usage are critical. In addition, our model is optimized in an end-to-end fashion without requiring post-training pruning, incorporating optimization sub-stages during the training, or increasing the overall finetuning time. We evaluate our method on Transformer-base models and show that effective finetuning is achievable along with base model parameters compression, which could be reduced by up to $40\%$ without significant loss in accuracy. In the next sections, we present our method and the empirical evidence of its effectiveness. We also provide a convergence proof of our method.

\section{Related Work}

\subsection{Low-Rank Adaptation}

Low-rank adaptation aims to tune LLMs with limited resources by updating a small number of parameters
by tuning injected layer modules \citep{pfeiffer2020adapterfusion,houlsby2019parameter}, embeddings \citep{lester2021power, li2021prefix} or training with a low-rank structure of the gradients \citep{refael2024adarankgrad}. One widely used method, LoRA \citep{hu2021lora}, tunes low-rank decomposed layers to avoid training cost overhead. However, LoRA keeps
the tuning layer shapes in the base model static without dynamic adjustments. Another work by \citep{he2022sparseadapter} dynamically adjusts tuning parameters during training, and \citep{zhang2023adalora} gradually reduces tuning parameters, but neither of them benefits the inference efficiency of the finetuned model.

\subsection{Finetuning with Pruning}
There are two main types of model pruning during finetuning: structured \citep{xia2022structured, zhao2024apt} and unstructured \citep{sanh2020movement}. Unstructured pruning prunes the most unimportant parameters in the model without any order, while structured pruning prunes entire blocks, rows, or columns in the weight matrices. Moreover, a post-training pruning method, i.e., proposed by \citep{frantar2023sparsegpt}, aims to prune finetuned models with limited extra costs but requires initialization from fully finetuned models. Compared to these methods, our approach focuses on accurately adapting the model to the target task while pruning 20-30\% of the base model parameters, with minimal accuracy loss. Additionally, the pruning process is integrated with model adaptation, ensuring no extra training time is required.

\subsection{Finetuning with Adaptive Pruning}

A recent work that closely aligns with our goal was proposed by \citep{zhaoapt}, where the base model parameters are pruned during the adaptor training. However, this method requires approximately five times more training time than full model finetuning, and the compression is achieved through sorting and binary searching on the weight blocks. Additionally, due to the adaptive rank of the low-rank adaptor weights, the memory usage during optimization approaches about 70\% of what is required for full model finetuning. In contrast, our simple method necessitates significantly less optimization memory, comparable to the LoRA consumption with a fixed rank, and does not impose any substantial additional training time burden compared to the full finetuning approach.

\section{Problem Formulation}
Assume we are given a pre-trained large language model $P_{\myvec{\Theta}} (\myvec{y}|\myvec{x})$ parametrized by $\myvec{\Theta}$ based on the Transformer architecture \citep{vaswani2017attention}. Our goal is to adapt this pre-trained model to downstream natural language understanding tasks, such as question-answering or sentiment analysis. The downstream task is represented by a training dataset of context-target pairs: $\mathcal{Z} = \{(\myvec{x}_i, \myvec{y}_i)\}_{i\in[N]}$, where $\myvec{x}_i$ is a sequence and $\myvec{y}_i$ is a target label. For example, in the question-answering task (QQP) in the GLUE benchmark \citep{wang2019superglue}, $\myvec{x}_i$ is a question, and $\myvec{y}_i$ is its corresponding answer. 

To finetune the whole model parameters (full finetuning), the model is initialized to pre-trained weights $\myvec{\Theta}_0$ and updated to $\myvec{\Theta}_0 + \Delta \myvec{\Theta}$ by repeatedly following the gradient updates to maximize the conditional language modeling objective:
\begin{equation}
    \max_{\myvec{\Theta}} \sum_{(\myvec{x},\myvec{y}) \in \mathcal{Z}} \sum_{t=1}^{|\myvec{y}|} \log (P_{\myvec{\Theta}} (y_t|\myvec{x},\myvec{y}_{<t})).
\end{equation}

In the low-rank adaptation method the task-specific parameter increment $\Delta \myvec{\Theta} = \Delta \myvec{\Theta} (\myvec{\Gamma})$ is further encoded by a much smaller-sized set of parameters $\myvec{\Gamma}$ with $|\myvec{\Gamma}|  \ll |\Theta_0|$. The task of finding $\Delta \Theta$ thus becomes optimizing over $\myvec{\Gamma}$ and not $\myvec{\Theta}$,
\begin{equation}
    \max_{\myvec{\Gamma}} \sum_{(\myvec{x},\myvec{y}) \in \mathcal{Z}} \sum_{t=1}^{|\myvec{y}|} \log (P_{\myvec{\Theta}_0 + \Delta \myvec{\Theta}(\myvec{\Gamma})} (y_t|\myvec{x},\myvec{y}_{<t})).
\end{equation}
While being beneficial for preserving the same base model for different tasks, this approach preserves non-useful information in the base model and still requires a forward pass through the large number of parameters in the base model during the inference.

In this work, we propose to add a \textit{gates} vector 
$\myvec{\omega}\in \{0,1\}^{1 \times d}$, parametrized by $\myvec{\Omega}$,
for the base model parameters additionally to the learned $\Delta \myvec{\Theta}$. Our approach implies a \textit{structured sparsity} on the base model \citep{wen2016learning} since we aim to exclude the whole columns in the weight matrices. Hence, the objective of the finetuning task becomes:
\begin{equation}
    \begin{aligned}
    \min_{\myvec{\Omega}} \max_{\myvec{\Gamma}} & \bigg( \sum_{(\myvec{x},\myvec{y}) \in \mathcal{Z}} \sum_{t=1}^{|\myvec{y}|} \log \left(P_{\myvec{\omega}_r \cdot (\myvec{\Theta}_0  + \Delta \myvec{\Theta}(\myvec{\Gamma})) \cdot \myvec{\omega}_c} (y_t|\myvec{x},\myvec{y}_{<t})\right) + \\
 &+ \lambda_1 \cdot \max(||\myvec{\omega}_r||_0, s) + \lambda_2 \cdot \max(||\myvec{\omega}_c||_0,s) \bigg),
    \end{aligned}
\end{equation}
where $||\cdot||_0$ is a zero norm of the gate parameters $\myvec{\omega}_r,\myvec{\omega}_c$ which multiply the rows and columns of the base model parameters such that $\myvec{\omega}_r \cdot (\myvec{\Theta}_0  + \Delta \myvec{\Theta}(\myvec{\Gamma})) \cdot \myvec{\omega}_c$. 

The parameter $\lambda$ represents the structured sparsity regularization magnitude and $s$ is a target sparsity ratio defined by the number of zero gates divided by the total number of gates in a gating vector.
To clarify, the structured sparsity is obtained by training two vectors $\myvec{\omega}_l,\myvec{\omega}_r$ where each element multiplies the entire row/column in a given weight matrix. Presenting such a sparsification mechanism helps to reduce the memory and time complexity in attention layers. 

An intriguing research question to consider is whether the sparsification applied to the base model weights is sufficient for adapting the model to the target task without the need to learn and add additional parameters, denoted as $\Delta \myvec{\Theta}(\myvec{\Gamma})$. In this scenario, only the target task head is optimized, while the base model is compressed through gating mechanisms to better fit this task. In this case, the simplified objective becomes: 
\begin{equation}
     \min_{\myvec{\Omega}} \sum_{(\myvec{x},\myvec{y}) \in \mathcal{Z}} \sum_{t=1}^{|\myvec{y}|} \log \left(P_{\myvec{\omega}_r \cdot \myvec{\Theta}_0 \cdot \myvec{\omega}_c} (y_t|\myvec{x},\myvec{y}_{<t})\right)
    + \lambda_1 \cdot \max(||\myvec{\omega}_r||_0, s) + \lambda_2 \cdot\max(||\myvec{\omega}_c||_0,s).
\end{equation}

Our empirical results show that, in this setup, it is also possible to reduce the model size and increase its efficiency while providing accurate predictions.

\section{The Method}

Consider the Transformer architecture \citep{vaswani2017attention} composed of $L$ blocks and each block consists
of a multi-head self-attention (MHA) layer and a feed-forward (FFN) layer. A MHA layer with $N_h$ heads takes an input $X$ and outputs:
\begin{equation*}
\text{MHA}(\myvec{X}) = \sum_{i=1}^{N_h} \text{Att}(\myvec{W}_q^{(i)}, (\myvec{W}_k^{(i)},\myvec{W}_v^{(i)},\myvec{W}_o^{(i)}, \myvec{X}),
\end{equation*}
where $\myvec{W}_q,\myvec{W}_k,\myvec{W}_v$ and $\myvec{W}_o$ refer to the query/key/value/output projection matrices, and $\text{Att}(\cdot)$ is an attention function. After attention head, the outputs are passed through the feed-forward layer, which consists of intermediate and output-projection layers, parameterized by $\myvec{W}_{mlp}^i$ and $\myvec{W}_{mlp}^o$:
\begin{equation*}
\text{FFN}(X)=\text{gelu}(\myvec{X}\myvec{W}_{mlp}^{i}) \cdot \myvec{W}_{mlp}^o.
\end{equation*}

Denote by $\myvec{W}_0 \in \mathbb{R}^{k \times d}$ a pre-trained weight matrix out of $\{ \myvec{W}_q,\myvec{W}_k,\myvec{W}_v, \myvec{W}_o, \myvec{W}_{mlp}^i, \myvec{W}_{mlp}^o\}$. As proposed by \citep{hu2021lora}, its update is constrained by representing the latter with a low-rank decomposition $$\myvec{W}_0 + \Delta \myvec{W} = \myvec{W}_0 + \myvec{W}_B \myvec{W}_A,$$ where $\myvec{W}_B \in \mathbb{R}^{k\times r}$, $\myvec{W}_A\in \mathbb{R}^{r\times d}$ and rank $r \ll \min(d,k)$. 

To enforce structured sparsity of the matrix $\myvec{W}_0$, we propose to multiply it by the learnable \textit{stochastic gates} vector $\myvec{\omega} \in \{0,1\}^{1 \times d}$ which is trained to converge into the binary representation. To achieve that, we learn a representation $\myvec{\mu} \in [-1,1] ^ {1 \times d}$ which is then converted to the approximate Bernoulli variables $\myvec{\omega}$, by utilizing a Gaussian-based relaxation of Bernoulli variables \citep{yamada2020feature,jana2023support}. The relaxation relies on the reparameterization trick \citep{miller2017reducing, figurnov2018implicit} and was demonstrated effective in several applications \citep{svirsky2023interpretable,lindenbaum2021l0,yang2023multi,lindenbaumtransductive}, aims to reduce the gradient estimates' variance. During the training, the conversion is done by adding random noise vector $\myvec{\epsilon} \in \mathbb{R} ^{1 \times d}$ to the shifted by scalar $0.5$ vector  $\myvec{\mu}$ and clipping the values by range $[0,1]$, 
\begin{equation}
    \myvec{\omega}(\myvec{\mu}) = \max(0, \min(1, 0.5 + \myvec{\mu} + \myvec{\epsilon})), 
\end{equation}
where each value in the vector $\myvec{\epsilon}$ is drawn from $\mathcal{N}(0, \sigma^2)$ and $\sigma=0.5$ is fixed throughout training. To encourage the model to produce sparse $\myvec{\omega}$ vector, it is trained with the regularization loss term constrained by the given sparsity ratio $s$: 
\begin{equation}
    \mathcal{L}_{\text{sparse}}(\myvec{\omega})=\max (||\myvec{\omega}||_0, s).
\end{equation}

Assuming that $\myvec{\omega}$ is a Bernoulli variable, we calculate its expected $\ell_0$ norm as follows:
\begin{align*}
& \mathbb{E} ||\myvec{\omega}||_0 = \frac{1}{d}\sum_j  \mathbb{P}(\omega_j > 0) = \frac{1}{d}\sum_j    \mathbb{P}(\mu_j + 0.5 + \epsilon_j > 0)= \frac{1}{d}\sum_j    (1 - \mathbb{P}(\mu_j + 0.5 + \epsilon_j \leq 0)) = \\
& = \frac{1}{d}\sum_j  \left(1 - \Phi \left(\frac{-\mu_j - 0.5}{\sigma} \right) \right) =\frac{1}{d}\sum_j 1 - \frac{1}{2} \left(1 + \erf \left(-\frac{\mu_j + 0.5 }{\sqrt{2} \sigma}\right)\right) = \\
& = \frac{1}{d} \sum_j \left( \frac{1}{2} - \frac{1}{2} \erf \left(-\frac{\mu_j + 0.5 }{\sqrt{2} \sigma}\right)\right).
\end{align*}

When using the $\cal{L}_{\text{sparse}}$ term, the model tries to sparsify the matrix $\myvec{W}_0$ and preserve only the parameters that are essential for the task learning. Finally, assuming a  latent representation is obtained by $\myvec{h}= \myvec{W}_0\myvec{x}$, our method's forward pass yields:
\begin{equation}
    \myvec{h = \left[ \myvec{\omega}_r \cdot \myvec{W}_0\cdot \myvec{\omega}_c \right] \cdot \myvec{x}},
\end{equation}
and it could be extended by presenting LoRA-style learnable matrices $\myvec{W}_B, \myvec{W}_A$:
\begin{equation}
    \myvec{h = \left[ \myvec{\omega}_r \cdot [\myvec{W}_0 + \myvec{W}_B\myvec{W}_A]\cdot \myvec{\omega}_c \right] \cdot \myvec{x}},
\end{equation}
\begin{figure*}[t]
\centering
\includegraphics[width=.55\linewidth]{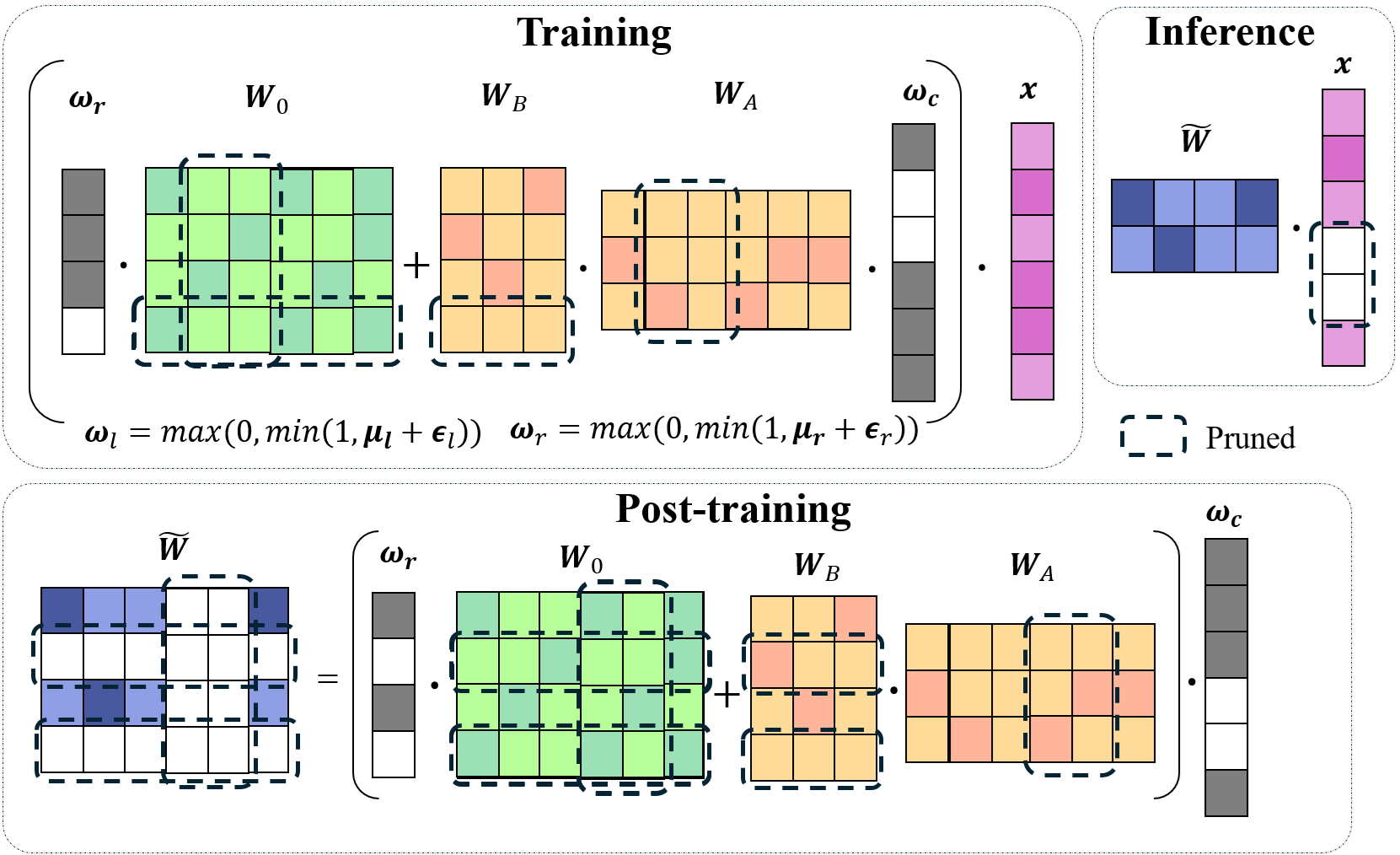}
\includegraphics[width=.4\linewidth]{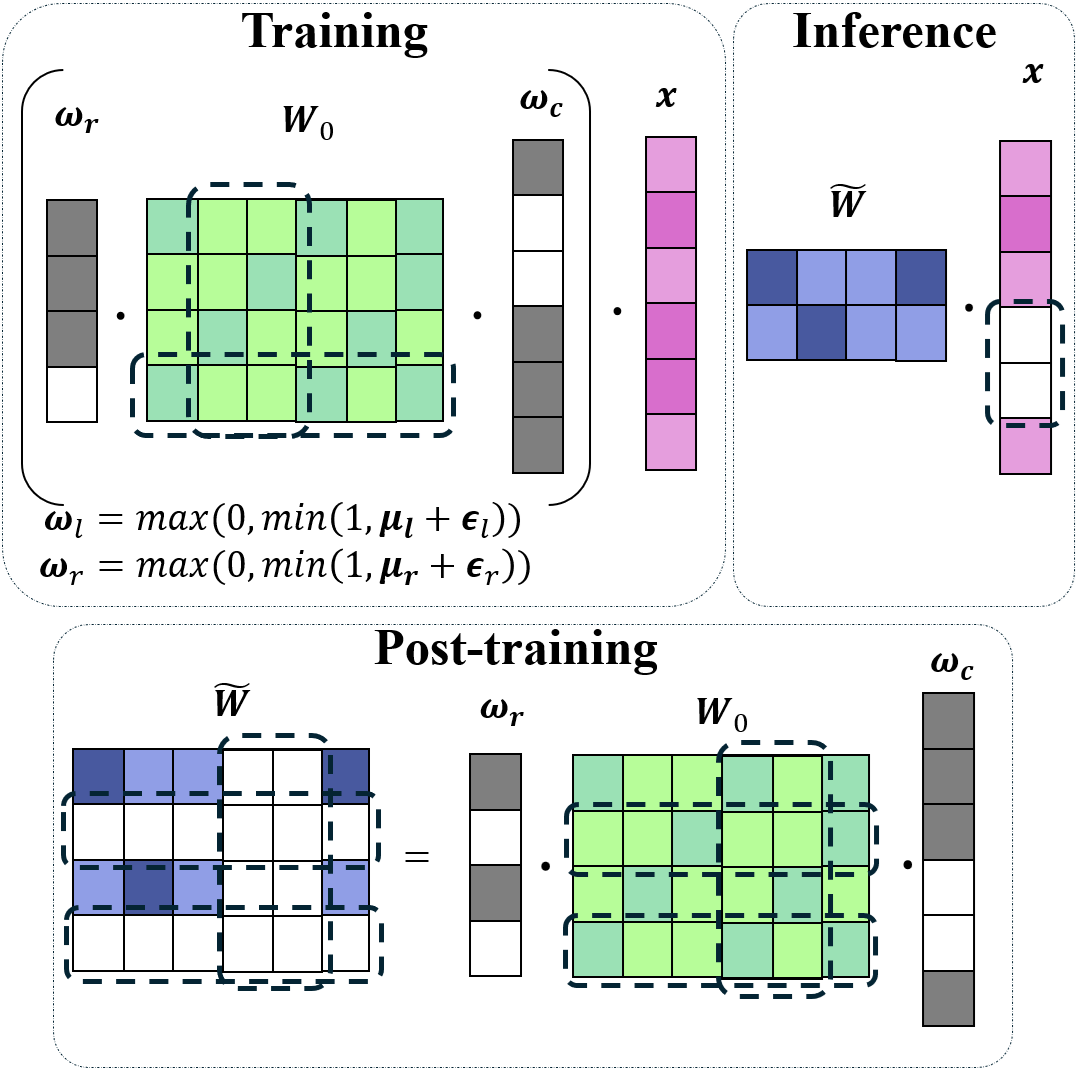} \\
\hspace{0.3in} \textbf{(a)} \hspace{3in} \textbf{(b)}
\caption{Two versions of our method: \textbf{(a)} In the first one, we train an adaptor with additional weights $\myvec{W}_A, \myvec{W}_B$. After training we compute the updated and pruned weight matrix $\Tilde{\myvec{W}} =  \myvec{\omega}_r \cdot (\myvec{W}_0 + \myvec{W}_B\myvec{W}_A)\cdot \myvec{\omega}_c$. \textbf{(b)} In the simplified version, the adaptor is based only on the trainable gates vectors $\myvec{\omega}_l, \myvec{\omega}_r$ that enforce structured sparsity.}

\label{fig:illusration}
\end{figure*}
where $\myvec{\omega}_r, \myvec{\omega}_c$ and $\myvec{W}_B, \myvec{W}_A$ are trainable parameters. Our method is depicted in Figure \ref{fig:illusration}. At the start of training, we initialize the vectors $\myvec{\omega}$  with all elements set to one.
Similarly to \citep{hu2021lora}, we use a random Gaussian initialization for $\myvec{W}_A$ and zero for $\myvec{W}_B$, so $\Delta \myvec{W} = \myvec{W}_B \myvec{W}_A$ is zero at the beginning of training. The total optimization objective with a hyperparameter $\lambda$ and task-specific loss term $\mathcal{L}_{\text{task}}$, i.e. cross-entropy, becomes: 
\begin{equation}
    \label{eq:loss} 
    \mathcal{L} = \mathcal{L}_{\text{task}} + \lambda \mathcal{L}_{\text{sparse}}.
\end{equation}
During training, we optimize the parameters $\myvec{\Omega}$ that multiply the matrices $\myvec{W}_q,\myvec{W}_k,\myvec{W}_v$, $\myvec{W}_o$ and $\myvec{W}_{mlp}$. In the extended version, we train also $\myvec{\Gamma}$ parameters that assemble the matrices $\myvec{W}_A, \myvec{W}_B$ for each adopted layer, as proposed by \citep{hu2021lora}.

Our method's simplicity allows us to train a small adaptor to the base model, achieving comparable or even improved accuracy. Additionally, we can significantly reduce the number of parameters in the base model with only a minor decrease in accuracy. Next, we present the empirical evaluation of our method. 

\section{Experiments}
\subsection{Experimental Setup and Datasets}

We assess the performance of our method on downstream tasks using the RoBERTa-base and RoBERTa-large models \citep{liu2019roberta}, and we utilize the widely recognized GLUE benchmark \citep{wang2019superglue}, as shown in Table \ref{tab:glue}. To simulate real-world conditions where data is typically scarce for finetuning tasks—due to the complex process of collecting ground truth labels—we limit the dataset to a maximum of 10,000 samples. Thus we use the small-scale datasets as-is (COLA, STSB, MRPC, RTE) and take the first 10K labeled samples from the large-scale datasets (MNLI, QQP, QNLI, SST2) in Table \ref{tab:glue_combined}. In addition, we conduct experiments on full SST2 and MNLI datasets (Table \ref{tab:sst2_mnli} and Figure \ref{fig:sparsity}). Each dataset has its own validation set on which all models are evaluated. We train all models by using Adam optimizer with decoupled weight decay regularization \citep{loshchilov2017decoupled} and optimize the $\myvec{\Omega}$ and $\myvec{\Gamma}$ separately with fixed learning rates for all tasks, $1e-3$ for $\myvec{\Omega}$ and $1e-4$ for $\myvec{\Gamma}$ parameters. We use NVIDIA A100 GPUs to train the models. We report the median accuracy value for each experiment over five random initialization seeds. The number of trainable parameters (TP) does not include the classifier head following the same setup as in previous works.

\begin{table}[t]
\caption{The GLUE benchmark datasets statistics.}
\label{tab:glue}
\begin{center}
\resizebox{.8\linewidth}{!}{
\begin{tabular}{cccccccccc} 
\hline
\textbf{Dataset} & MNLI & QQP & QNLI & SST2 & COLA & STSB & MRPC & RTE \\
\hline
\textbf{Samples} & 392,702 & 363,846 & 104,743 & 67,349 & 8,551 & 5,749 & 3,668 & 2,490 \\
\hline
\end{tabular}}
\end{center}
\end{table}

\begin{table*}[t]
\caption{Finetuning accuracy on GLUE benchmark datasets. We present the best accuracy results achieved by FineGates and the accuracy obtained with sparsity constraint $s \geq 10\%$, $s \geq 20\%$. The number of removed parameters is shown in {\color{olive}{olive}}, and the relative change in accuracy is depicted in {\color{gray}{gray}} compared to full fine-tuning. The results for some methods are missing because the corresponding baselines were not evaluated on these sub-sampled datasets (SST2, MNLI, QNLI and QQP). }
\label{tab:glue_combined}
\begin{center}
\resizebox{.99\linewidth}{!}{
\begin{tabular}{ll|llll|llll} 
\hline
\multicolumn{2}{c}{} & \multicolumn{4}{c}{Small} & \multicolumn{4}{c}{Sub-sampled Large}\\
Method & TP $\downarrow$ & CoLA & STS-B & MRPC & RTE & SST2 & MNLI & QNLI & QQP \\ 
\hline 
\multicolumn{10}{c}{\textbf{Roberta-Base}} \\
\hline
Full Finetune & 125M & 63.6 & 90.9 & 90.2 & 80.5 & 92.8 & 81.4 & 87.7 & 85.2  \\
\hline
LoRA($r=8$) \citep{hu2021lora} & 0.3M & 63.4 & \textbf{91.5} & 89.7 & \textbf{86.6} & \textbf{94.3} & \textbf{81.3} & \textbf{89.8} & \textbf{85.0} \\ 
LoRA-FA \citep{zhang2023lora} & 1.8M & 63.6 & 89.6 & \textbf{90.0}  & 67.9 & - & - & - & - \\
VeRA \citep{kopiczko2023vera} & 0.04M & 65.6 & 90.7 & 89.5 & 78.7 & - & - & - & - \\
FineGates & 0.17M & \textbf{65.7} & 90.8 &\textbf{90.0} & 83.4 & 94.0 & \textbf{81.3} & 89.1 & 84.9 \\
FineGates, $s \geq 10\%$ \color{olive}{(-12M)} & 0.17M & 65.2 ({\color{gray} +1.6\%}) & 90.7({\color{gray} -0.2\%}) & 89.2 ({\color{gray} -1\%}) & 81.3 ({\color{gray} +0.8\%})& 94.0 ({\color{gray} +1.2\%})& 80.6 ({\color{gray} -0.8\%}) & 88.8 ({\color{gray} +1.1\%}) & 84.3 ({\color{gray} -0.9\%})\\ 
FineGates, $s \geq20\%$ \color{olive}{(-25M)} & 0.17M & 61.4 ({\color{gray} -2.2\%}) & 90.5 ({\color{gray} -0.4\%})& 87.7 ({\color{gray} -2.5\%})& 78.0 ({\color{gray} -2.5\%})& 93.8 ({\color{gray} +1\%})& 79.4 ({\color{gray} -2\%})& 87.6 ({\color{gray} -0.1\%})& 83.7 ({\color{gray} -1.5\%})\\
\hline
\multicolumn{10}{c}{\textbf{Roberta-Large}} \\
\hline
Full Finetune & 355M & 68.0 & 92.3 & 90.9 & 86.6 & 93.5 & 85.9 & 92.4 & 85.8 \\
\hline
LoRA($r=8$) \citep{hu2021lora} & 0.8M & 68.2 & \textbf{92.6} & 90.9 & 87.4 & \textbf{96.2} & \textbf{87.2} & \textbf{93.2} & \textbf{86.5} \\
LoRA-FA \citep{zhang2023lora} & 3.7M & 68 & 92 & 90 & 86.1 & - & - & - & - \\
LoRA-XS \citep{balazy2024lora} & 0.06K &  68.5 &  92.2 & \textbf{91.2} & 89.5 & - & - & -& -\\
VeRA \citep{kopiczko2023vera} & 0.06M & 68.0 & 91.7 &  90.9 & 85.9 & - & - & - & - \\
FineGates & 0.4M & \textbf{71.4} & 92.1  & 90.9 & \textbf{90.2} & \textbf{96.2} & 86.2 & 92.4 & 86.1 \\
FineGates, $s \geq10\%$ \color{olive}{(-35M)} & 0.4M & 69.5 ({\color{gray} +1.5\%})& 91.8 ({\color{gray} -0.5\%})& 90.4 ({\color{gray} -0.5\%})& 89.2 ({\color{gray} +2.6\%})& 96.1 ({\color{gray}  +2.6\%})& 85.0 ({\color{gray} -0.9\%})& 92.1 ({\color{gray} -0.3\%})& 86.1 ({\color{gray} +0.3\%})\\ 
FineGates, $s \geq20\%$ \color{olive}{(-70M)} & 0.4M & 68.1 ({\color{gray} +0.1\%})& 91.3 ({\color{gray} -1\%})& 90.2 ({\color{gray} -0.7\%})& 87.0 ({\color{gray} +0.4\%})& 95.3 ({\color{gray} +1.8\%})& 85.0 ({\color{gray} -0.9\%})& 92.0 ({\color{gray} -0.4\%})& 85.2 ({\color{gray} -0.6\%})\\ 
\hline
\end{tabular}}
\end{center}
\end{table*} 
\begin{table}[h!]
\begin{center}
    \caption{Finetuning accuracy on MNLI, SST2 full datasets. The results for VeRA and LoRA-XS methods are missing because they were not evaluated on MNLI dataset.}
    \label{tab:sst2_mnli}
    \resizebox{.35\linewidth}{!}{
    \begin{tabular}{ll|cc} 
    \hline
    Method & TP $\downarrow$ & SST2 & MNLI \\ 
    \hline 
    \multicolumn{4}{c}{\textbf{Roberta-Base}} \\
    \hline
    Full Finetune & 125M & 94.8 & 87.6  \\
    \hline
    LoRA($r=8$) & 0.3M &  \textbf{95.1} & \textbf{87.5} \\ 
    LoRA-FA \citep{zhang2023lora} & 1.8M & 94.8 & 86.8 \\
    VeRA \citep{kopiczko2023vera} & 0.04M & 94.6 &- \\
    FineGates & 0.17M &  95.0 & 85.1 \\
    \hline
    \multicolumn{4}{c}{\textbf{Roberta-Large}} \\
    \hline
    Full Finetune & 355M & 96.4 & 90.2 \\
    \hline
    LoRA($r=8$) & 0.8M & 96.2 & \textbf{90.6} \\ 
    LoRA-FA \citep{zhang2023lora} & 3.7M &  96.0 & 90.1 \\
    LoRA-XS \citep{balazy2024lora} &  0.06M & \textbf{96.3} & - \\
    VeRA \citep{kopiczko2023vera} & 0.06M & 96.1 & - \\
    FineGates & 0.4M & \textbf{96.3} & 88.5\\
    \hline
    \end{tabular}
}
\end{center}
\end{table}

\subsection{Baselines}
We compare our method to full finetuning, LoRA \citep{hu2021lora}, LoRA-FA \citep{zhang2023lora}, VeRA \citep{kopiczko2023vera}, LoRA-XS \citep{balazy2024lora} methods. During the full finetuning, the model is initialized to the pre-trained weights and biases, and all model parameters undergo gradient updates. In the LoRA baseline, only the $\myvec{W}_B, \myvec{W}_A$ are learned while all base model layers are frozen. We train FineGates+LoRA and LoRA methods with the same $r=8$ on the subsampled datasets: SST2, MNLI, QNLI, and QQP.

Moreover, in Section \ref{sec:apt}, we compare our method and recently proposed APT \citep{zhaoapt}. While APT provided promising results by adaptively pruning the base model, our method is beneficial in two key components: (1) we learn the gates for the base model weights jointly with the finetuning task objective, (2) our model obtains comparable results without pruning attention heads and without additional distillation loss. However, encouraged by the APT \citep{zhaoapt} model and its predecessor CoFi \citep{xia2022structured}, we plan to extend our framework to be able to prune attention heads to increase the speedup and overall model sparsity.

\subsection{Accuracy Results}

We present the accuracy results in Table \ref{tab:glue_combined}. We report the overall (matched and mismatched) accuracy for MNLI, Matthew’s correlation for CoLA, Pearson correlation for STS-B, and accuracy for other tasks. 
From Table \ref{tab:glue_combined}, it could be seen that our model is comparable with LoRA and outperforms the full finetuning on average while being applied to the Roberta-Base base model for datasets CoLA, STS-B, MRPC, and RTE. 

Additionally, as shown in the TP column, our method achieves a significant reduction in the number of trainable parameters compared to LoRA for both the Roberta-Base and Roberta-Large models which is only $\sim 0.14\%$ of total parameters in the base model. Furthermore, our approach not only reduces the trainable parameter count but also compresses the base model itself, resulting in a parameter reduction of $10-20\%$ of parameters in the base model, all while maintaining an insignificant loss in accuracy (two last rows in the table). This highlights the efficiency and effectiveness of our method in balancing compression and performance. Our model outperforms the full finetune baseline and performs comparatively as well as the LoRA baseline for SST2, MNLI, QNLI, and QQP datasets. In addition, our model provides significant parameter reduction with accuracy close to LoRA and better than full finetune for these datasets.

We conduct also an evaluation of our method on full SST2 and MNLI datasets. The results are presented in Table \ref{tab:sst2_mnli}, where it can be seen that our method is comparable to other methods when applied in the Roberta-Base model and provides the best results for the SST2 dataset when applied to the Roberta-Large base model.

\subsection{Sparsification Results}
\begin{figure*}[h!]
\centering
\includegraphics[width=.32\linewidth]{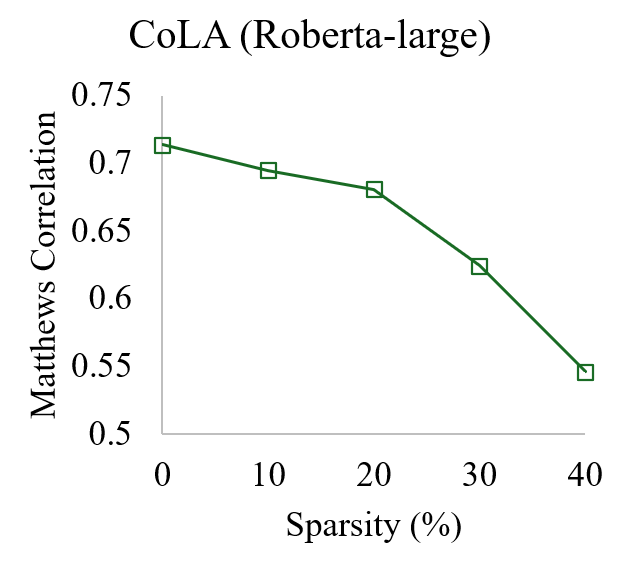}
\includegraphics[width=.32\linewidth]{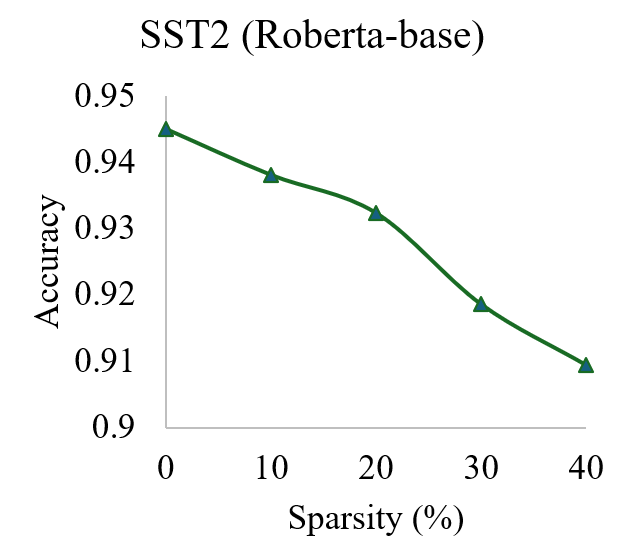}
\includegraphics[width=.32\linewidth]{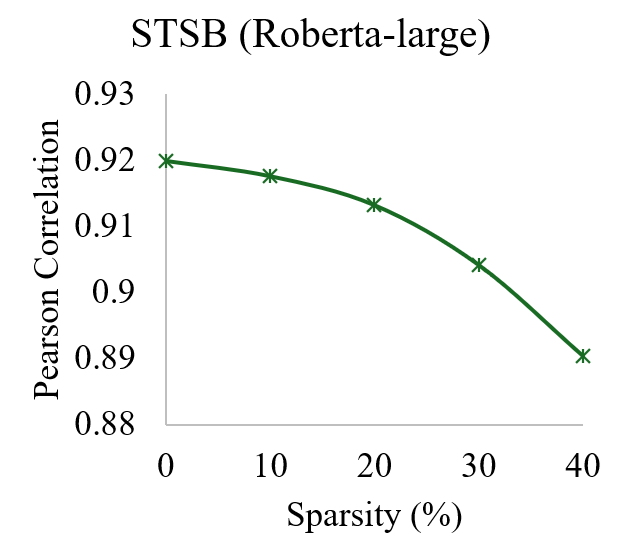}

\caption{Sparsification-Accuracy trade-off measured on CoLA, SST2, and STSB datasets. Our model provides $>\mathbf{40\%}$ of structured sparsity while sacrificing only $\mathbf{4\%}$ of accuracy compared to the model without sparsification on the SST2 dataset where we train $\myvec{\omega}_r, \myvec{\omega}_c$ with total $166K$ parameters. On CoLA the method reduces up to $20\%$ of parameters without significant loss in accuracy, and $40\%$ on the STSB dataset with only $3\%$ drop in accuracy.
} 
\label{fig:sparsity}

\end{figure*}

We present the sparsification results of FineGates in Figure \ref{fig:sparsity}. We report accuracy measurements for each sparsity level. It could be seen that our method can remove up to $20\%$ of parameters without significant loss in Matthews correlation for CoLA dataset, up to $40\%$ of parameters trained on the SST2 dataset with only a loss of $4\%$ in accuracy, and up to $40\%$ for STSB dataset with a loss of only $~3\%$ in Pearson Correlation metric.

\subsection{FineGates Modifications}
In Table \ref{tab:glue_combined_ablation}, we present three modifications of our mode: FineGates w/o $\myvec{W}_{mlp}$ that does not adapt intermediate and output projections for RoBERTa layers,  $\text{FineGates}+$LoRA adds low-rank matrices $\myvec{W}_{B},\myvec{W}_{A}$ with $r=8$ and $\text{FineGates}+$LoRA w/o $\myvec{W}_{mlp}$ is similar but without intermediate and output projections.
\begin{table*}[ht]
\caption{Ablation study of FineGates. The evaluation is done for the next versions of FineGates: (1) training without gates on $\myvec{W}_{mlp}$ matrices, (2) training with additional low-rank parameters $\myvec{W}_B \myvec{W}_A$, (3) training with low-rank parameters but without $\myvec{\omega}, \myvec{W}_B\myvec{W}_A$ parameters for $\myvec{W}_{mlp}$ matrices.}
\label{tab:glue_combined_ablation}
\begin{center}
\resizebox{.8\linewidth}{!}{
\begin{tabular}{ll|llll|llll} 
\hline
\multicolumn{2}{c}{} & \multicolumn{4}{c}{Small} & \multicolumn{4}{c}{Sub-sampled Large}\\
Method & TP $\downarrow$ & CoLA & STS-B & MRPC & RTE & SST2 & MNLI & QNLI & QQP \\ 
\hline 
\multicolumn{10}{c}{\textbf{Roberta-Base}} \\
\hline
FineGates & 0.17M & 65.7 & 90.8  & 90.0 & 83.4 & 94.0 & 81.3 & 89.1 & 84.9 \\
FineGates w/o $\myvec{W}_{mlp}$ & 0.04M & 65.1 & 90.9  & \textbf{90.4} & 80.5 & 94.2 & 81.0 & 89.2 & 84.7 \\
FineGates$+$LoRA & 1.4M & \textbf{66.5} &\textbf{91.2} & 90.2 & \textbf{83.8} & \textbf{94.3} & 81.4 & 89.4 & 84.8 \\
FineGates$+$LoRA w/o $\myvec{W}_{mlp}$ & 0.6M & 65.8 & 90.7 & 89.7 & 82.3 & \textbf{94.3} & \textbf{81.7} & 89.7 & \textbf{85.2} \\
\hline
\hline
\multicolumn{10}{c}{\textbf{Roberta-Large}} \\
\hline
FineGates & 0.4M & \textbf{71.4} & 92.1  & 90.9 & 90.2 & \textbf{96.2} & 86.2 & 92.4 & 86.1 \\
FineGates w/o $\myvec{W}_{mlp}$ & 0.1M & 70.5 & 92.0  & 91.4 &\textbf{90.6} & 96.1 & 86.6 & 92.3 & 86.4 \\
FineGates$+$LoRA & 3.8M & 70.1 & 92.2 & 90.6 & 88.1 & 95.4 & 86.1 & 92.1 & 85.5 \\
FineGates$+$LoRA w/o $\myvec{W}_{mlp}$ & 1.7M & 69.9 & \textbf{92.6} & \textbf{91.9} & 88.5 & \textbf{96.2} & \textbf{87.2} & \textbf{92.7} & \textbf{86.7} \\
\hline
\hline
\end{tabular}}
\end{center}
\end{table*}

\subsection{Comparison to the APT method}
\label{sec:apt}
To compare our method against the recently proposed APT model, we conduct experiments on MRPC dataset. First, we obtain the results for APT dataset with default parameters provided by the authors in the code. We observe, that the rank varies between values 8 and 64, hence, we conduct an experiment, where FineGates is trained with additional LoRA matrices with different ranks: $\{16, 32, 64 \}$ on the MRPC dataset with target sparsity fixed at 40\%. The results are presented in Table \ref{tab:mrpc_apt}. FineGates achieves performance comparable to the APT model but without the extensive pruning applied by APT. In our method, the sparsity rate is determined by reducing the number of attention weight dimensions after multiplication with the learned gates. However, when calculating the sparsity rate, we do not include parameter reductions resulting from pruning attention heads or reducing embedding dimensions. Hence, to obtain the same sparsity level as the APT method, our gates are required to be more sparse for matrices $W_q, W_k, W_v, W_o$ and $W_{mlp}$ than in the APT model.

\begin{table}
\begin{center}
    \caption{Comparison of FineGates+LoRA against APT method on MRPC dataset with Roberta-Base backbone.}
    \label{tab:mrpc_apt}
    \resizebox{.45\linewidth}{!}{
    \begin{tabular}{ll|cc} 
    \hline
    Method & Max $r$  & Sparsity & Accuracy \\ 
    \hline
    APT & 64 & $40\%$ & 87.9 \\
    FineGates+LoRA & 16 & $40\%$ & 84.6 \\
    FineGates+LoRA & 32 & $40\%$ & 84.8 \\
    FineGates+LoRA & 64 & $40\%$ & 86.5 \\
    \hline
    \end{tabular}
}
\end{center}
\end{table}

\subsection{Inference Speedup}

\textbf{Matrix multiplication speedup} \hspace{0.1in} 
We now assess the potential for speed improvements in latency offered by our model. Generally, to avoid the dependence of model speedup on device and software specifications, it is preferable to present the number multiply-accumulate (MAC) operations. However, the reduction in MACs is directly proportional to the column reduction. Instead, we present the real clock-time improvement measured by a wall clock on a specific device. To achieve that, we measure the multiplication time of $\myvec{W}^T \cdot \myvec{X}$ and compare it to the time of $(\myvec{W}^T \cdot \myvec{\omega}) (\myvec{X}  \cdot \myvec{\omega})$. We use a single weight $\myvec{W} \in \mathbb{R}^{1024 \times 1024}$ with an input tensor $\myvec{X} \in \mathbb{R}^{16 \times 1024}$ without bias, and repeat this operation 100K times on a CPU device (Intel(R) Core(TM) i9-12900H).
In Figure \ref{fig:times}(a), we present the measured time in milliseconds and add the relative time reduction in percentages as labels for each point. We note that the indexing operation adds a small computation overhead presented at the most left point at the zero sparsity level.

\textbf{Overall model speedup} \hspace{0.1in} 
Next, we measure the times for a single inference epoch on the CoLA validation set. We train FineGates with sparsity levels up to $40\%$  and measure the total inference time on a single NVIDIA GeForce RTX 3080 GPU. We repeat this experiment 10 times for each sparsity level and report the measured time in Figure \ref{fig:times}(b) with relative time factor (RTF) as labels for each point, which is computed relatively to the zero sparsity level as inference time without sparsity divided by the processing time at the given sparsity level. Combining the fact that our model achieves up to $20-30\%$ structured sparsity (Table \ref{tab:glue_combined} with the time reduction presented in Figure \ref{fig:times}, we conclude that our method leads to reduced training and inference time while maintaining high performance.

\begin{figure}[t]
\centering
\includegraphics[width=.4\linewidth]{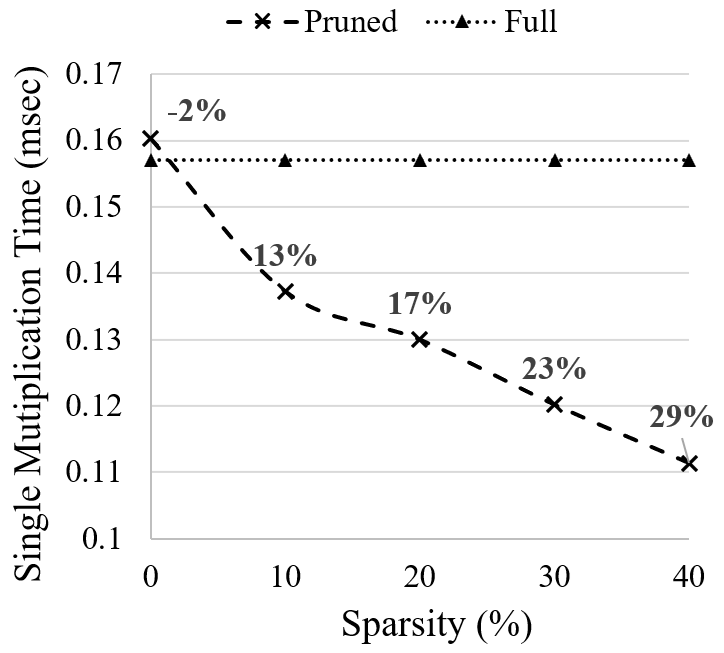}
\includegraphics[width=.45\linewidth]{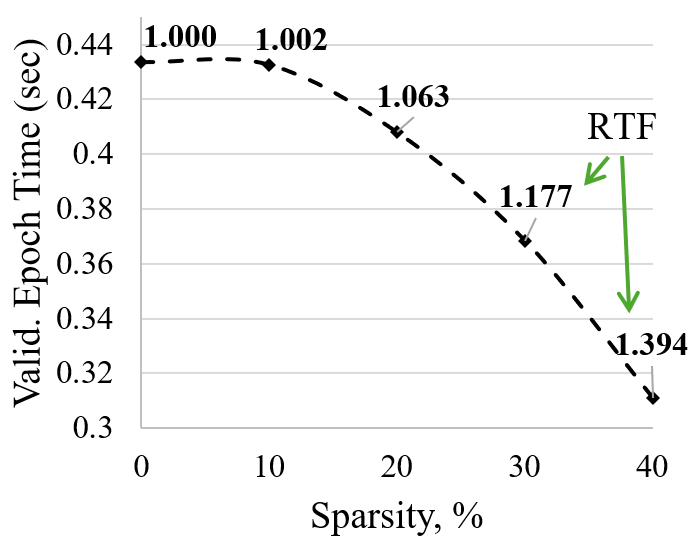}\\
(a) \hspace{2.5in} (b)
\caption{(a) Measuring relative time reduction in multiplication $(\myvec{W}^T \cdot \myvec{\omega}) (\myvec{X}  \cdot \myvec{\omega})$ compared to full matrices multiplication $\myvec{W}^T\myvec{X}$. We measure CPU time by repeating the operation 100K times and reporting the average time (vertical line) for each sparsity level (horizontal line). (b) Measuring inference time for a single validation epoch with varying sparsity levels.
} 
\label{fig:times}
\end{figure}

\section{Convergence}
In this section, we present a convergence proof for our method. This theoretical justification is necessary because including random noise in the gating mechanism could introduce challenges to the training process and affect convergence.

\begin{proposition}[Convergence of FineGates]
Suppose, $f\equiv\mathcal{L}_{\text{task}},$ is an $L$-smooth non-convex function that is bounded by $M$, then minimizing\footnote{By using the vanilla SGD} the whole objective $\mathcal{L}$ (FineGates) is guaranteed to converge to a stationary point.
\end{proposition}
\begin{proof}
The relaxed objective function (\ref{eq:loss}) can be generally rewritten as,
$\mathcal{L}(\mathbf{W},\myvec{\omega})\equiv f\left(\mathbf{W}\cdot\myvec{\omega}\right) +\lambda h(\myvec{\omega}),$ where $\mathcal{L}_{\text{task}}\left(\mathbf{W},\myvec{\omega}\right)\equiv f\left(\mathbf{W}\cdot\myvec{\omega}\right),$ and $\mathcal{L}_{\text{sparse}}\left(\myvec{\omega}\right)\equiv h(\myvec{\omega})
$ is the regularization term.

Consider two points $ (\mathbf{W}_1, \myvec{\omega}_1)  \neq  (\mathbf{W}_2, \myvec{\omega}_2)$ and define $f(\mathbf{W}\cdot\myvec{\omega})\equiv F(\mathbf{W},\myvec{\omega})$. We aim to show that,
$$
\|\nabla F(\mathbf{W}_1, \myvec{\omega}_1) - \nabla F(\mathbf{W}_2, \myvec{\omega}_2)\| \leq L^\prime \|(\mathbf{W}_1, \myvec{\omega}_1) - (\mathbf{W}_2, \myvec{\omega}_2)\|
$$
for some constant \( L^\prime \).
 To that end, we first calculate  
 the gradient of \( F(\mathbf{W}, \myvec{\omega}) \) with respect to \( \mathbf{W} \),
$
\nabla_{\mathbf{W}} F(\mathbf{W}, \myvec{\omega}) = \nabla f(\mathbf{W} \cdot \myvec{\omega}) \cdot \myvec{\omega}^\top,
$
and with respect to \( \myvec{\omega} \) is,
$
\nabla_{\myvec{\omega}} F(\mathbf{W}, \myvec{\omega}) = \mathbf{W}^\top \nabla f(\mathbf{W} \cdot \myvec{\omega}).
$

First, for the gradient with respect to \( \mathbf{W} \), we have,
\begin{align*}
    \|&\nabla_{\mathbf{W}} F(\mathbf{W}_1, \myvec{\omega}_1) - \nabla_{\mathbf{W}} F(\mathbf{W}_2, \myvec{\omega}_2)\| = \|\nabla f(\mathbf{W}_1 \cdot \myvec{\omega}_1) \cdot \myvec{\omega}_1^\top - \nabla f(\mathbf{W}_2 \cdot \myvec{\omega}_2) \cdot \myvec{\omega}_2^\top\|
\end{align*}
By the Lipschitz continuity of \( \nabla f \), this can be bounded by,
$$
L \|(\mathbf{W}_1 \cdot \myvec{\omega}_1 - \mathbf{W}_2 \cdot \myvec{\omega}_2)\| \|\myvec{\omega}_1^\top\|.
$$
Since \( 0 \leq \omega_i \leq 1 \), the norm of \( \myvec{\omega}_1 \) is bounded, implying that the term is bounded by a constant times \( \|(\mathbf{W}_1, \myvec{\omega}_1) - (\mathbf{W}_2, \myvec{\omega}_2)\| \).

For the gradient with respect to \( \myvec{\omega} \), we have,
\begin{align*}
\|&\nabla_{\myvec{\omega}} F(\mathbf{W}_1, \myvec{\omega}_1) - \nabla_{\myvec{\omega}} F(\mathbf{W}_2, \myvec{\omega}_2)\|= \|\mathbf{W}_1^\top \nabla f(\mathbf{W}_1 \cdot \myvec{\omega}_1) - \mathbf{W}_2^\top \nabla f(\mathbf{W}_2 \cdot \myvec{\omega}_2)\|.
\end{align*}
Again, using the Lipschitz continuity of \( \nabla f \) and the boundedness of \( \mathbf{W} \), this term is similarly bounded by a constant times \( \|(\mathbf{W}_1, \myvec{\omega}_1) - (\mathbf{W}_2, \myvec{\omega}_2)\| \).

Thus, both gradient terms are Lipschitz continuous, and we conclude that \( F(\mathbf{W}, \myvec{\omega}) \) has a Lipschitz continuous gradient with respect to both \( \mathbf{W} \) and \( \myvec{\omega} \). 
Now, recall that $h(\myvec{\omega})$ is relaxed into $h(\myvec{\mu})\equiv\frac{1}{d}\sum^d_{j=1} \left(\frac{1}{2} - \frac{1}{2} \erf\left(-\frac{\mu_{j} + 0.5}{\sqrt{2}\sigma}\right) \right)$ which is continuously differential bounded function (where the sum of the probabilities that the gates $\left\{\omega_i\right\}_{i=1}^d$ are active, or $\sum_{i \in[d]} \mathbb{P}\left(\omega_i>0\right)\equiv\sum_{j=1}^d \Phi\left(\frac{\mu_j}{\sigma}\right)$, with $\Phi$ stands for the standard Gaussian CDF (please refer to \citep{yamada2020feature})). Thus, finally, $\mathcal{L}(\mathbf{W},\myvec{\omega})$, is a sum of continues smooth functions, meaning it holds the condition for the vanilla SGD convergence \citep{ketkar2017stochastic}.
\end{proof}

\section{Conclusion}
In this work, we propose a novel finetuning method that enforces structured sparsity on the weights of a base large language model (LLM). Our approach enables the removal of up to $20-30\%$ parameters in the attention matrices during the adaptation. We empirically compared our method against LoRA and full finetuning baselines on the GLUE benchmark with a limited number of train samples to 10K per task and analyzed the speedup provided by our method. In contrast to most low-rank adaptation methods, our method compresses the base model's weight dimensions during finetuning while working within a limited sample budget. In future research, we plan investigate additional methods to further compress LLM models during finetuning and address multi-task finetuning \citep{feng2024mixture}.

\bibliographystyle{plainnat}
\bibliography{refs}

\end{document}